\documentclass[acmsmall, screen]{acmart}
\AtBeginDocument{%
  }

\setcopyright{acmlicensed}
\copyrightyear{2026}
\acmYear{2026}
\acmDOI{XXXXXXX.XXXXXXX}

\acmJournal{TIST}

\usepackage{multirow}
\usepackage{booktabs}
\usepackage{algorithm}
\usepackage{algpseudocode}
\usepackage{amsthm}
\usepackage{textcomp}
\usepackage{xcolor}
\newtheorem{proposition}{Proposition}
\newtheorem{remark}{Remark}

\begin{document}

\title{TL++: Accuracy and Privacy Preserving Traversal Learning for Distributed Intelligent Systems}

\author{Erdenebileg Batbaatar}
\affiliation{%
  \institution{Neouly}
  \city{Seoul}
  \country{Republic of Korea}
}
\email{erdenebileg11@gmail.com}

\author{Young Yoon}
\authornote{Corresponding author}
\affiliation{%
  \institution{Department of Computer Engineering, Hongik University}
  \city{Seoul}
  \country{Republic of Korea}
}
\email{young.yoon@hongik.ac.kr}

\renewcommand{\shortauthors}{Batbaatar and Yoon}

\begin{abstract}
Distributed intelligent systems increasingly need to train across data silos without centralizing raw data. Federated learning keeps data local, but heterogeneous partitions can degrade accuracy and require repeated full-model exchange. Split learning reduces communication through cut-layer activations, but standard protocols do not generally recover pooled mini-batch gradients and may expose activations and gradients in plaintext. We present TL++, a two-mode traversal-learning framework that constructs virtual batches across nodes to recover centralized mini-batch gradient behavior under explicit synchronization assumptions.
Base mode exchanges cut-layer activations and gradients instead of full models. Secure mode secret-shares each cut-layer activation and gradient between an orchestrator and a non-colluding helper, so neither server alone observes plaintext cut-layer tensors. This protection is limited to a semi-honest, non-colluding two-server setting; labels and output values used for loss remain visible to the orchestrator. In the lightweight secure path evaluated here, exactness requires the sharewise server path to be linear or affine; nonlinear server operations require nonlinear MPC or are approximate.
We formalize TL++, analyze communication and computation costs, and evaluate it against federated- and split-learning baselines on CIFAR-10 and BioGPT/\allowbreak{}PubMedQA using full fine-tuning and LoRA. On CIFAR-10, TL++ base cut~1 and exact secure cut~3 achieve $91.41 \pm 0.19\%$ and $90.93 \pm 0.17\%$ accuracy, outperforming the strongest measured non-TL++ baseline by more than 12 percentage points. TL++ base cut~1 also reduces per-step communication payload by $13.1\times$ relative to full-model synchronization. PubMedQA results similarly favor TL++ under both tuning settings, though they remain descriptive without paired trajectories. Overall, TL++ approaches centralized-level utility while reducing communication and supporting scoped activation-level secret sharing.
\end{abstract}




\maketitle

\section{Introduction}
\label{sec:intro}

Edge and distributed intelligent systems---including clinical monitors, autonomous fleets, and industrial sensors---generate sensitive data that could improve model robustness if pooled centrally. Centralization, however, raises privacy and bandwidth costs and can conflict with data-sovereignty regulations such as GDPR~\cite{european_union_gdpr_2016} and HIPAA~\cite{us_congress_hipaa_1996}, as well as broader privacy and federated-learning concerns~\cite{kairouz_advances_2021,li_review_2020}. A central challenge for distributed artificial intelligence (AI) is therefore to combine \emph{centralized-gradient equivalence}, \emph{communication efficiency}, and \emph{security of intermediate computations}. Here, ``accuracy losslessness'' means producing the same mini-batch gradient update that centralized training would compute on the same sampled batch, not guaranteeing identical final test accuracy across all seeds or implementations.

No existing paradigm satisfies all three requirements simultaneously. Federated Learning (FL)~\cite{mcmahan_communication-efficient_2017} keeps raw data on-device and aggregates local model updates, but falls short on each dimension. Under non-independent and identically distributed (non-IID) data, local/global gradient divergence degrades convergence with label skew~\cite{zhao_federated_2018,karimireddy_scaffold_2020,wang_tackling_2020,acar_federated_2021,li_model-contrastive_2021}. Full-parameter exchange grows costly with model size; compression and quantization reduce bytes but may compound approximation error under heterogeneity~\cite{konecny_federated_2016,lin_deep_2018,alistarh_qsgd_2017,bernstein_signsgd_2018}. Finally, shared model updates can reveal training samples~\cite{zhu_deep_2019,zhao_idlg_2020,geiping_inverting_2020}, while secure aggregation~\cite{bonawitz_practical_2017} adds overhead and does not protect intermediate split-learning computations.

Clustered Federated Learning (CFL)~\cite{sattler_clustered_2019} improves personalization by partitioning incompatible clients into specialized models, but it remains an FL-family method: clients exchange model updates, the target is no longer the pooled centralized objective, and intermediate split-learning activations or gradients are not protected. CFL is therefore complementary to TL++ rather than a substitute for its virtual-batch and activation-level security design.

Split Learning (SL)~\cite{gupta_distributed_2018,vepakomma_split_2018} improves communication by transmitting cut-layer activations rather than full parameters, but it still processes one client's samples per forward pass, precluding cross-participant gradient aggregation and large-batch benefits under heterogeneity~\cite{li_federated_2020}. It also sends cut-layer activations and gradients in plaintext, exposing them to gradient inversion~\cite{zhu_deep_2019}, model inversion~\cite{he_model_2019}, and optimization-based reconstruction even at deep cuts~\cite{pasquini_unleashing_2021}.

Traversal Learning (TL)~\cite{batbaatar_traversal_2025} resolves the accuracy limitation by coordinating mini-batch construction across distributed nodes. TL assembles \emph{virtual batches} spanning multiple participants, enabling gradient aggregation mathematically equivalent to centralized training on pooled data, while retaining SL's communication advantage of transmitting only cut-layer activations. TL satisfies \emph{accuracy losslessness} and \emph{communication efficiency}, but transmits activations and gradients in plaintext, offering no privacy guarantee against a semi-honest or compromised server.

TL++ extends TL with two modes. Base mode reproduces TL's plaintext protocol for trusted environments. Secure mode adds additive-secret-sharing-based multi-party computation (MPC)~\cite{damgard_multiparty_2012} through a non-colluding helper server, following common PPML designs~\cite{mohassel_secureml_2017,mohassel_aby3_2018,wagh_securenn_2019}, so cut-layer activations and gradients are split rather than exposed. The helper provides privacy separation, not optimization: it owns no training data, does not define the objective, and only processes the second masked share. Without this non-colluding holder, the orchestrator would receive or reconstruct plaintext tensors, reducing secure TL++ to the trusted base setting.

The cut also does not itself cause accuracy loss: it only assigns the upper-model forward pass and corresponding backpropagation to the server partition. If that partition evaluates the same $f_{\text{server}}$ and returns the centralized cut-layer gradient, the node-side update remains centralized-gradient equivalent by the chain rule. Thus, $f_{\text{server}}$ may be nonlinear in base mode. The linearity condition is narrower: the lightweight secure protocol can evaluate shares independently without changing the result only when the sharewise path is linear or affine. For a cut activation $a=a^{(1)}+a^{(2)}$, a linear/affine map $L$ satisfies $L(a^{(1)}+a^{(2)})=L(a^{(1)})+L(a^{(2)})$ once the bias is handled once or shared consistently, and its Jacobian is independent of the hidden plaintext. ReLU, pooling, normalization, and softmax generally violate this identity; secure TL++ can still support them with secure nonlinear protocols, otherwise sharewise evaluation is an approximation. The current additive-sharing protocol therefore makes only some cuts exact; it does not require linear server networks in all TL++ architectures.

This integration raises three corresponding challenges.

\begin{enumerate}
    \item \textbf{Accuracy:} The server partition must return the same cut-layer gradient that centralized backpropagation would produce. Plain delegation of upper-layer backpropagation does not break accuracy losslessness; approximating or incorrectly evaluating the server-side computation under sharing does.
    \item \textbf{Communication:} Secret sharing doubles cut-path activation traffic and adds helper/server coordination, so the useful operating point depends on cut depth and synchronization cost.
    \item \textbf{Security:} The protocol must keep intermediate activations and cut-layer gradients shared while reconstructing only values needed for loss computation.
\end{enumerate}

Table~\ref{tab:comparison} positions TL++ relative to prior paradigms. The main contributions are as follows.

\newcommand{\cmark}{\ensuremath{\checkmark}}
\newcommand{\xmark}{\texttimes}
\newcommand{\pmark}{$\sim$}
\begin{table}[t]
    \caption{Comparison of distributed learning paradigms. \textbf{(A)}~Centralized-gradient equivalence under non-IID data. \textbf{(C)}~Communication efficiency vs.\ full parameter transmission.
    \textbf{(S)}~Protection for intermediate values under the stated threat model.
    \checkmark~satisfied; \texttimes~not satisfied; $\sim$~partial. For TL++ secure, the \checkmark{} entries assume semi-honest non-colluding servers and an exact additive-share evaluation path. In the lightweight implementation evaluated here, exactness holds when the operations applied independently to additive shares are linear/affine. Nonlinear server-side operations require additional secure nonlinear protocols; otherwise the corresponding secure run is approximate.}
    \label{tab:comparison}
    \small
    \setlength{\tabcolsep}{4pt}
    \begin{tabular}{lp{5cm}ccc}
    \toprule
    \textbf{Method} & \textbf{Transmission} & \textbf{(A)} & \textbf{(C)} & \textbf{(S)} \\
    \midrule
    FL~\cite{mcmahan_communication-efficient_2017}
      & Model params & \xmark & \xmark & \xmark \\
    CFL~\cite{sattler_clustered_2019}
      & Model params / clustered models & \pmark & \xmark & \pmark \\
    FL+Compress~\cite{lin_deep_2018,wangni_gradient_2018,alistarh_qsgd_2017,bernstein_signsgd_2018}
      & Compressed grads & \xmark & \pmark & \xmark \\
    FL+secure aggregation~\cite{bonawitz_practical_2017}
      & Encrypted params & \xmark & \xmark & \pmark \\
    \cmidrule(lr){1-5}
    SL~\cite{gupta_distributed_2018,vepakomma_split_2018}
      & Activations & \xmark & \cmark & \xmark \\
    SL+differential privacy
      & Noisy activations~\cite{dwork_differential_2006} & \pmark & \cmark & \pmark \\
    SL+homomorphic encryption
      & Encrypted activations~\cite{sav_cure_2024} & \xmark & \pmark & \cmark \\
    SplitFed~\cite{thapa_splitfed_2022}
      & Activations & \pmark & \cmark & \xmark \\
    \cmidrule(lr){1-5}
    TL~\cite{batbaatar_traversal_2025}
      & Activations & \cmark & \cmark & \xmark \\
    \midrule
    \textbf{TL++ (base)}
      & Activations & \cmark & \cmark & \xmark \\
    \textbf{TL++ (secure)}
      & Secret-shared activations; exact when sharewise server path is linear/affine & \cmark & \cmark & \cmark \\
    \bottomrule
    \end{tabular}
\end{table}

\begin{itemize}
    \item \textbf{A two-mode traversal-learning framework.} TL++ supports a base mode for trusted deployments and a secure mode for privacy-sensitive deployments. Both modes use traversal-learning virtual batches; secure mode adds a non-colluding helper server and additive-share message formats.

    \item \textbf{A condition for exact additive-share evaluation.} We show that TL virtual-batch forward and backward passes are exact over additive shares when the operations evaluated independently on shares are linear or affine. This is a protocol condition, not an architectural mandate: nonlinear $f_{\text{server}}$ is valid in base mode and in secure mode with secure nonlinear evaluation. In the current CNN implementation, exact secure evaluation corresponds to cut~3; secure cut~1 and cut~2 are reported as approximations because nonlinear server layers are evaluated without an added secure nonlinear protocol.

    \item \textbf{Activation-level privacy with limited reconstruction.} TL++ splits cut-layer activations and gradients between the orchestrator and helper. Only output values and label-dependent signals needed for loss and backpropagation are reconstructed; labels, metadata, collusion, and malicious deviations remain outside this guarantee unless additional mechanisms are added.

    \item \textbf{Communication analysis across cut points.} We analyze the cost of applying MPC to cut-layer activations rather than full model updates. Secure mode roughly doubles cut-path activation/gradient traffic relative to base TL, but deeper cuts substantially reduce the payload.

    \item \textbf{Implementation and evaluation.} We instantiate TL++ in a VGG-style~\cite{simonyan_very_2015} CIFAR-10~\cite{krizhevsky_learning_2009} setting and a BioGPT/\allowbreak{}PubMedQA~\cite{luo_biogpt_2022,jin_pubmedqa_2019} low-rank adaptation (LoRA)~\cite{hu_lora_2021} setting, covering image and biomedical natural language processing (NLP) workloads under the same split-training design.
\end{itemize}

The remainder of the paper reviews related work, presents the TL++ protocol, evaluates utility and communication cost, and discusses privacy assumptions and deployment limits.

\section{Related Work}
\label{sec:related}

We map prior work against TL++'s three concerns: \emph{accuracy losslessness} (A), \emph{communication efficiency} (C), and \emph{security of intermediate computations} (S). Table~\ref{tab:comparison} summarizes the landscape.

\subsection{Federated Learning}

FL trains without centralizing raw data by aggregating local updates, but it only partially meets the three concerns. On~(A), non-IID distributions can drive divergence between local and global objectives, with degradation tied to label skew~\cite{zhao_federated_2018}. SCAFFOLD~\cite{karimireddy_scaffold_2020}, FedProx~\cite{li_fedprox_2020}, FedNova~\cite{wang_tackling_2020}, FedDyn~\cite{acar_federated_2021}, MOON~\cite{li_model-contrastive_2021}, FedBN~\cite{li_fedbn_2021}, adaptive federated optimizers~\cite{reddi_adaptive_2021}, and personalized methods such as Ditto~\cite{li_ditto_2021} mitigate objective inconsistency, drift, or feature shift, but do not generally recover the exact pooled centralized mini-batch gradient used by TL++. On~(C), full-parameter exchange grows with model size~\cite{konecny_federated_2016,kairouz_advances_2021}. On~(S), model updates can reveal training samples~\cite{zhu_deep_2019}, and secure aggregation~\cite{bonawitz_practical_2017} adds overhead without protecting intermediate split-learning computations.

Clustered and multi-task variants address cases where heterogeneity reflects multiple legitimate tasks. CFL~\cite{sattler_clustered_2019} recursively partitions clients using update-direction similarity, improving personalization when a single pooled objective is inappropriate. This clarifies the scope of TL++: a single TL++ job assumes compatible labels, model semantics, and conditional distributions; CFL-style compatibility discovery could be used upstream to decide which clients should share a TL++ model.

\subsection{Communication-Efficient Distributed Learning}

Communication-efficient learning reduces bytes through compression, lower synchronization frequency, or network partitioning. Deep gradient compression~\cite{lin_deep_2018}, sparsification~\cite{wangni_gradient_2018}, QSGD~\cite{alistarh_qsgd_2017}, sign-based optimization~\cite{bernstein_signsgd_2018}, FedPAQ~\cite{reisizadeh_fedpaq_2020}, and local SGD~\cite{stich_local_2019} reduce FL traffic but add approximation or staleness that can compound non-IID divergence, leaving~(A) and~(S) unresolved. SL and TL~\cite{batbaatar_traversal_2025} instead send cut-layer activations, with savings that grow when the cut is earlier. Asynchronous and buffered FL reduce idle time~\cite{xie_asynchronous_2019,nguyen_fedbuff_2022} but remain model-update protocols. SplitFed~\cite{thapa_splitfed_2022} parallelizes client-side split computation and aggregates client models, yet without cross-client virtual batches it does not recover pooled centralized-gradient equivalence under non-IID data. TL++ adds additive sharing, so its net benefit depends not only on activation size but also on node-model synchronization, node-gradient sharing, helper relay, serialization, and synchronization cadence; efficient operating points are cut-dependent.

\subsection{Split Learning}

SL, proposed by Gupta and Raskar~\cite{gupta_distributed_2018} and Vepakomma et al.~\cite{vepakomma_split_2018}, partitions a network at a cut layer: clients send activations, the server completes the forward pass and loss, then returns cut-layer gradients. Although used in privacy-sensitive domains such as healthcare~\cite{vepakomma_split_2018}, standard SL processes one client's samples per pass, preventing cross-participant gradient aggregation and large-batch benefits under heterogeneity~\cite{li_federated_2020}. Its privacy is also limited: gradients can reconstruct inputs~\cite{zhu_deep_2019}, activations are vulnerable to model inversion~\cite{he_model_2019}, and malicious servers can optimize reconstructions even at deep cuts~\cite{pasquini_unleashing_2021}. NoPeek~\cite{vepakomma_nopeek_2020} reduces input--smashed-data dependence but not the cryptographic non-observability targeted by TL++ secure mode. TL~\cite{batbaatar_traversal_2025} supplies virtual batches; SplitFed~\cite{thapa_splitfed_2022} supplies parallelism; TL++ adds protection for exposed activations and gradients.

\subsection{Privacy-Preserving Machine Learning}

Three main cryptographic approaches protect machine learning computations, but each falls short of satisfying all three concerns simultaneously in this setting. Differential privacy (DP)~\cite{dwork_differential_2006} injects calibrated noise into model updates, providing formal guarantees but with a privacy-utility trade-off that is unfavorable at high privacy budgets, particularly under non-IID data~\cite{kairouz_advances_2021}---degrading~(A) alongside~(S). Homomorphic encryption (HE)~\cite{gentry_fully_2009} supports computation on ciphertext. Secure neural-network inference systems such as CryptoNets~\cite{gilad_bachrach_cryptonets_2016}, GAZELLE~\cite{juvekar_gazelle_2018}, and Delphi~\cite{mishra_delphi_2020} show how HE can be combined with circuit-based or protocol-level optimizations, but they primarily target private inference rather than virtual-batch training. Recent work has applied HE directly to SL: CURE~\cite{sav_cure_2024} encrypts server-side model parameters using the CKKS approximate-arithmetic HE scheme~\cite{cheon_homomorphic_2017}, protecting intermediate activations and labels from a semi-honest server while retaining client-side plaintext computation. CURE achieves strong~(S) guarantees and partially satisfies~(C) by limiting encryption to the server-side portion only, but the HE computational overhead remains substantial for deep networks with many training iterations, limiting practical scalability and compromising~(A) when accuracy-utility trade-offs are applied. MPC~\cite{damgard_multiparty_2012} enables multiple parties to jointly compute functions over private inputs without revealing them. Among MPC primitives, additive secret sharing is especially well-suited to the linear operations that dominate gradient aggregation, as linear functions compose correctly over shares without reconstruction (see Section~\ref{sec:pre}). Bonawitz et al.~\cite{bonawitz_practical_2017} demonstrate practical secure aggregation for FL using additive secret sharing, protecting model updates from a semi-honest server, but this does not extend to intermediate split-learning activations and gradients. TL++ adopts additive secret sharing on cut-layer activations and gradients directly, confining reconstruction to the output layer where loss computation requires it---satisfying~(S) without sacrificing (A) or~(C), and aligning with two- and three-party PPML protocols that rely on non-collusion or honest-majority assumptions~\cite{mohassel_secureml_2017,mohassel_aby3_2018,wagh_securenn_2019}. Unlike HE-based approaches, additive secret sharing introduces no ciphertext expansion for linear operations, and its linear overhead scales with activation size rather than model depth; more general private neural-network systems such as SecureNN and ABY3 support nonlinear layers by adding interaction and mixed protocols~\cite{wagh_securenn_2019,mohassel_aby3_2018}. This makes TL++ complementary to, rather than a replacement for, secure aggregation in federated settings.

\subsection{Traversal Learning}

TL was designed to satisfy~(A) and~(C): unlike FL, it does not average independently trained local models, and unlike SL, it does not process clients as separate gradient updates. Instead, TL assembles virtual batches across nodes and executes one forward--backward pass, producing updates equivalent to centralized mini-batch training on pooled data while leaving raw data local. Batbaatar et al.~\cite{batbaatar_traversal_2025} report gains over FL and SL across IID, non-IID, text, and medical datasets, especially under label skew.

TL++ uses TL as base mode and adds secure mode with a non-colluding helper: activations and gradients are split so neither server sees plaintext intermediate values. When the sharewise path is linear/affine, the secure update is numerically identical to base mode while satisfying~(S) at the cost of additional activation/share traffic. To our knowledge, TL++ is the first deployable framework to combine accuracy-lossless gradient aggregation, split-learning-level communication, and configurable protection of intermediate training signals.

\section{Preliminaries}
\label{sec:pre}

\subsection{Problem Formulation}

Consider a distributed system of $N$ nodes, where node $i$ holds a local dataset $\mathcal{D}_i = \{(\mathbf{x}^{(i)}_j, y^{(i)}_j)\}_{j=1}^{n_i}$ of $n_i$ labeled samples. Let $n = \sum_{i=1}^{N} n_i$ denote the total number of samples across all nodes, and let $\mathcal{D} = \bigcup_{i=1}^{N} \mathcal{D}_i$ denote the (virtual) pooled dataset. Raw data never leaves its originating node.

The global training objective is to minimize the empirical loss over the pooled dataset:
\begin{equation}
  \min_{\boldsymbol{\theta}} \; \mathcal{L}(\boldsymbol{\theta}) = \frac{1}{n} \sum_{i=1}^{N} \sum_{j=1}^{n_i} \ell\!\left(f\!\left(\mathbf{x}^{(i)}_j;\,\boldsymbol{\theta}\right),\, y^{(i)}_j\right),
  \label{eq:global_loss}
\end{equation}
where $f(\cdot\,;\boldsymbol{\theta})$ is the neural network with parameters $\boldsymbol{\theta}$ and $\ell$ is a per-sample loss (e.g., cross-entropy).

A standard mini-batch SGD step on a batch $\mathcal{B} \subseteq \mathcal{D}$ of size $B$ computes:
\begin{equation}
  \boldsymbol{\theta} \leftarrow \boldsymbol{\theta} - \eta \cdot \frac{1}{B} \sum_{(\mathbf{x},y) \in \mathcal{B}} \nabla_{\boldsymbol{\theta}}\,\ell\!\left(f(\mathbf{x};\boldsymbol{\theta}),\,y\right),
  \label{eq:sgd}
\end{equation}
where $\eta$ is the learning rate. Centralized training draws $\mathcal{B}$ uniformly from $\mathcal{D}$, naturally mixing samples from all nodes in every update.

Under non-IID data, $P(\mathcal{D}_i)$ may differ from $P(\mathcal{D})$, so local-objective methods such as standard FL can compute gradients biased relative to Equation~\eqref{eq:global_loss}, degrading convergence~\cite{zhao_federated_2018,karimireddy_scaffold_2020}. Here, \emph{accuracy losslessness} is a gradient-level requirement: the distributed protocol should produce the same update as Equation~\eqref{eq:sgd} on the same uniformly drawn batch $\mathcal{B}$, under matched initialization, batch order, optimizer state, and arithmetic. Final test accuracy may still vary across seeds, hardware, and regularization.

\subsection{Split Learning Architecture}

SL partitions the network $f$ at a designated \emph{cut layer} $c$ into two sub-networks: a \emph{node model} $f_{\text{node}}$ comprising layers $1$ through $c$, and a \emph{server model} $f_{\text{server}}$ comprising layers $c+1$ through the output. Each node $i$ holds a local copy of $f_{\text{node}}$ with parameters $\boldsymbol{\theta}_{\text{node}}$, while the server maintains $f_{\text{server}}$ with parameters $\boldsymbol{\theta}_{\text{server}}$.

For a sample $\mathbf{x}$, the forward pass proceeds as follows. The node computes the \emph{cut-layer activation}:
\begin{equation}
  \mathbf{a} = f_{\text{node}}(\mathbf{x};\,\boldsymbol{\theta}_{\text{node}}),
  \label{eq:activation}
\end{equation}
and transmits $\mathbf{a}$ to the server. The server completes the forward pass, computing the output $\hat{y} = f_{\text{server}}(\mathbf{a};\,\boldsymbol{\theta}_{\text{server}})$, evaluates the loss $\ell(\hat{y}, y)$, and backpropagates to obtain the
\emph{cut-layer gradient}:
\begin{equation}
  \mathbf{g} = \frac{\partial\,\ell}{\partial\,\mathbf{a}}.
  \label{eq:cutgrad}
\end{equation}
The server returns $\mathbf{g}$ to the node, which uses it to backpropagate through $f_{\text{node}}$ and update $\boldsymbol{\theta}_{\text{node}}$.

The communication cost per sample is $O(|\mathbf{a}|)$, proportional to the activation dimension at the cut layer, rather than $O(|\boldsymbol{\theta}|)$ as in FL. This advantage grows as the cut layer is placed earlier in the network or as the total model size increases~\cite{gupta_distributed_2018,vepakomma_split_2018}.

\subsection{Secure Multiparty Computation}

MPC lets parties jointly evaluate $h(\mathbf{v}_1, \ldots, \mathbf{v}_k)$ over private inputs without any party learning more than its own input and the output~\cite{yao_protocols_1982,damgard_multiparty_2012}.

\paragraph{Additive Secret Sharing.}
TL++ employs two-party additive secret sharing as its MPC primitive~\cite{damgard_multiparty_2012,mohassel_secureml_2017}. Formally, additive sharing is usually defined over a finite field or ring. In the implementation, tensors are represented with finite-precision arithmetic, and masks are sampled from the corresponding numerical domain. For readability, we write the operation over real-valued tensors: for a value $v \in \mathbb{R}^d$, the \emph{secret sharing} operation samples a random mask $r$ from the implementation domain and defines two \emph{shares}:
\begin{equation}
  v_1 = r, \qquad v_2 = v - r,
  \label{eq:sharing}
\end{equation}
so that $v_1 + v_2 = v$ up to the chosen arithmetic representation. Under the ideal finite-domain sharing model, each share is individually independent of $v$; reconstruction requires both shares.

\paragraph{Linearity of Additive Secret Sharing.}
The key property that makes additive secret sharing compatible with gradient aggregation is its linearity: for any values $u, v \in \mathbb{R}^d$ with shares $(u_1, u_2)$ and $(v_1, v_2)$, and any scalar $\alpha \in \mathbb{R}$:
\begin{align}
  (u + v)_k &= u_k + v_k, \quad k \in \{1, 2\}, \label{eq:lin_add}\\
  (\alpha\, v)_k &= \alpha\, v_k, \quad k \in \{1, 2\}. \label{eq:lin_scale}
\end{align}
That is, sums and scalar multiples of shared values can be computed by each party independently on its own shares, without communication and without reconstruction. Reconstruction of the result requires only the standard $v_1 + v_2 = v$ step.

This property makes the lightweight secure TL++ path exact when it processes shares without reconstructing the cut activation. The claim is not architectural: in base mode $f_{\text{server}}$ may be any differentiable subnetwork, and secure mode may also use nonlinear $f_{\text{server}}$ if paired with secure nonlinear computation. The condition concerns only sharewise additive masking. For a linear map $L$, $L(\mathbf{a}_1+\mathbf{a}_2)=L(\mathbf{a}_1)+L(\mathbf{a}_2)$; affine biases can be added once or shared consistently. Thus, with a linear/affine sharewise path and no optional noise, masking introduces no approximation and the returned cut-layer gradient matches centralized backpropagation.

Nonlinear operations evaluated independently on shares are different. ReLU, pooling, normalization, and attention softmax generally violate $f(\mathbf{a}_1+\mathbf{a}_2)=f(\mathbf{a}_1)+f(\mathbf{a}_2)$, and their backward masks or Jacobians depend on the combined plaintext activation. Exact secure use therefore requires secure comparison, lookup, garbled-circuit, polynomial-approximation, or hybrid-MPC computation; otherwise the nonlinear sharewise path is an explicit approximation. For ReLU, the centralized operation $\operatorname{ReLU}(a^{(1)}+a^{(2)})$ differs from $\operatorname{ReLU}(a^{(1)})+\operatorname{ReLU}(a^{(2)})$; if $a^{(1)}=1$ and $a^{(2)}=-2$, the former is 0 and the latter is 1. Backpropagation has the same problem because ReLU masks depend on the combined plaintext. Any loss of centralized-gradient equivalence therefore comes from approximating nonlinear operations under masking, not from the server-side subnetwork itself.

\paragraph{Security Model.}
Following common two-server MPC and PPML assumptions~\cite{damgard_multiparty_2012,mohassel_secureml_2017}, TL++ assumes a \emph{semi-honest} (honest-but-curious) adversary: each party follows the protocol faithfully but attempts to infer private values from observed messages. Under this model, additive secret sharing guarantees that neither the orchestrator (holding $v_1$) nor the helper server (holding $v_2$) can recover $v$ without the other's share, provided they do not collude. The guarantee is scoped to shared cut-layer activations and gradients; it does not hide labels, reconstructed outputs, timing, batch-membership metadata, or values intentionally revealed for loss computation. The non-collusion assumption between orchestrator and helper is a standard deployment assumption in two-party MPC~\cite{damgard_multiparty_2012, bonawitz_practical_2017} and is discussed further in Section~\ref{sec:discussion}.

\section{The Proposed TL++}
\label{sec:method}

\subsection{System Architecture}

TL++ has $N$ \emph{nodes}, one \emph{orchestrator}, and one \emph{helper server}, following non-colluding PPML deployment patterns~\cite{mohassel_secureml_2017,mohassel_aby3_2018,wagh_securenn_2019}. Node $i$ holds $\mathcal{D}_i$, a copy of $f_{\text{node}}$, and local forward/backward computation. The orchestrator maintains $f_{\text{server}}$, constructs virtual batches, and distributes parameters and sample indices.

The server model remains a modeling choice: in base mode it may be nonlinear, while the exact lightweight secure implementation restricts only the operations evaluated independently on additive shares. A linear/affine sharewise path is exact; nonlinear paths require secure nonlinear protocols or are approximate. The helper maintains the corresponding sharewise path and participates in additive-secret-sharing MPC~\cite{damgard_multiparty_2012}, receiving activation shares from nodes and exchanging partial outputs and label-dependent output-gradient signals with the orchestrator. Its role is privacy separation: it holds the second share so no single server can reconstruct cut-layer activations or gradients. If the orchestrator is trusted with those tensors, the helper is omitted and TL++ runs in base mode.

Figure~\ref{fig:architecture} shows the topology, per-batch messages, and secure flow. Nodes split each cut-layer activation between the orchestrator and helper. Messages are: \textbf{A}, local sample counts; \textbf{B}, model parameters and batch indices; \textbf{C}, activation shares; \textbf{D}, cut-layer gradient shares; and \textbf{E}, node-model gradients for aggregation. A dedicated inter-server Phase~2 exchange lets the helper send its partial output, the orchestrator reconstruct predictions for loss, and the orchestrator return the label-dependent output-gradient signal (Section~\ref{sec:secure_proto}). Panel~C shows activation sharing, independent sharewise propagation, and the shared gradient relation $\mathbf{G}_m = \mathbf{G}^{(1)}_m + \mathbf{G}^{(2)}_m$.

TL++ operates in two modes sharing this architecture:
\begin{itemize}
  \item \textbf{Base mode:} Nodes transmit plaintext activations to the orchestrator only; no helper server is required. This reproduces TL's protocol exactly and is suited to trusted deployments.
  \item \textbf{Secure mode:} Nodes split activations using Equation~\eqref{eq:sharing} and send one share to each server. The helper is needed only to protect cut-layer tensors from any single semi-honest server. Linear/affine server computations are exact on shares; nonlinear computations require secure nonlinear evaluation or become approximate. The orchestrator reconstructs output values for label-based loss and distributes the resulting output-gradient signal; cut-layer activations and gradients remain shared under the semi-honest non-collusion assumption.
\end{itemize}

\begin{figure}[ht]
  \centering
     \includegraphics[width=\linewidth]{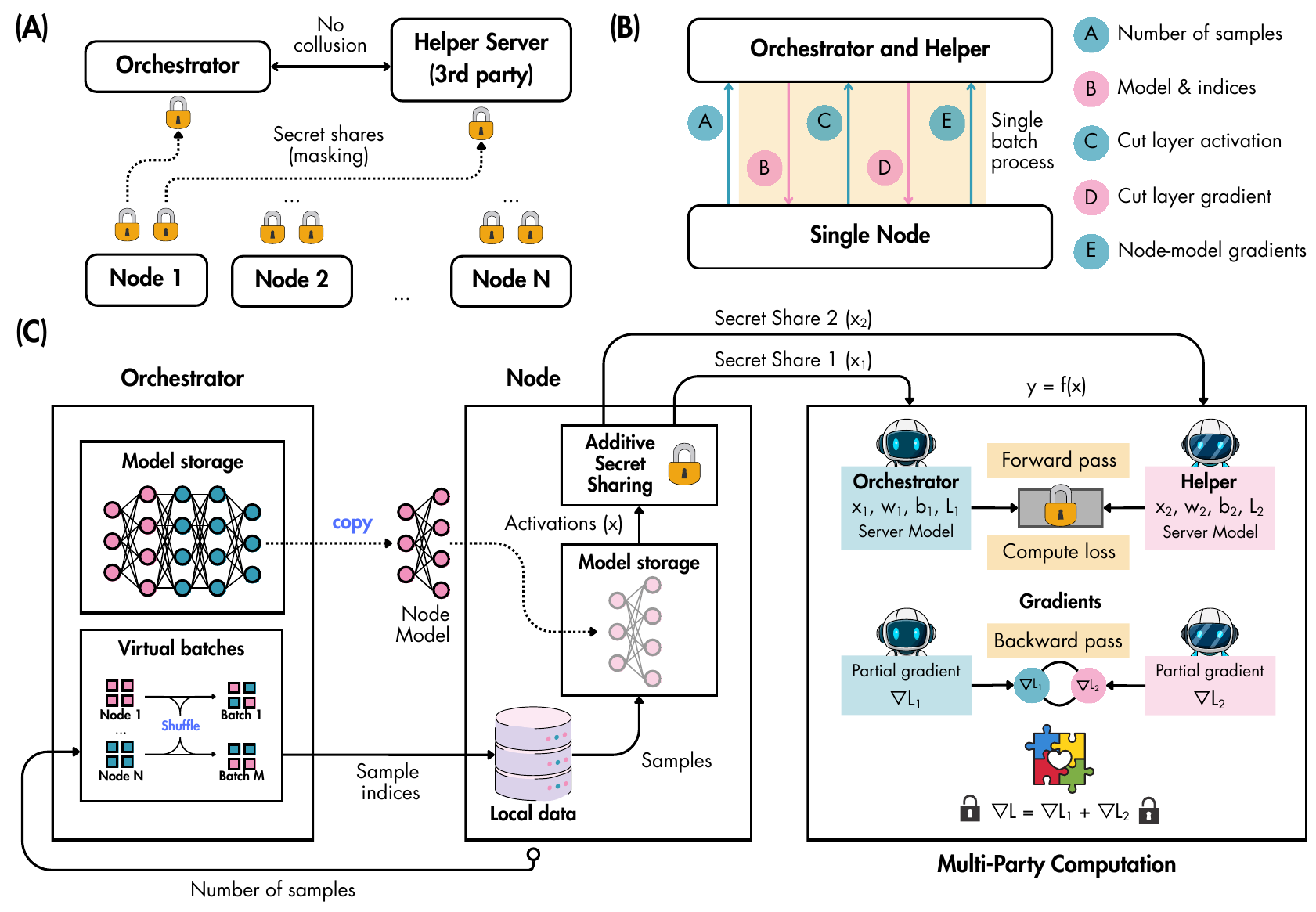}%
  {%
  }
  \caption{TL++ architecture.
  \textbf{(A)}~Three-party topology: $N$ nodes send additive secret shares of cut-layer activations to the orchestrator and a non-colluding helper server, so neither party observes plaintext intermediate values. \textbf{(B)}~Per-batch communication: messages~A--E denote sample count reporting, model parameter and index distribution, cut-layer activation shares, cut-layer gradient shares, and node-model gradient aggregation, respectively. An additional inter-server coordination channel (not shown) is used in secure mode for output reconstruction. \textbf{(C)}~Secure protocol: nodes split activations into shares $(x_1, x_2)$ via additive secret sharing; orchestrator and helper independently propagate their shares through the linear/affine server model used by the exact secure protocol; the orchestrator reconstructs outputs for loss computation and shares the output-gradient signal needed to form partial cut-layer gradients
  $\mathbf{G}_m = \mathbf{G}^{(1)}_m + \mathbf{G}^{(2)}_m$.}
  \label{fig:architecture}
  \Description{Three-panel diagram showing (A) three-party system topology with orchestrator, helper server, and N nodes connected via secret shares, (B) the per-batch SL message flow labeled with messages A through E, and (C) the secure additive-secret-sharing-based forward and backward pass protocol with partial gradients.}
\end{figure}

\subsection{Model Architecture and Cut-Layer Constraint}
\label{sec:model}

TL++ experiments use a custom VGG-style CNN~\cite{simonyan_very_2015} designed for CIFAR-10 ($32 \times 32$ red--green--blue (RGB) images)~\cite{krizhevsky_learning_2009}. The full network consists of three convolutional blocks followed by fully connected classification layers:
\begin{itemize}
    \item \textbf{Block~1:} Two Conv2D layers ($3 \to 64 \to 64$ channels, $3\times3$ kernels, padding~1), each followed by ReLU activation, then $2\times2$ MaxPool --- output: $64 \times 16 \times 16$.
    \item \textbf{Block~2:} Two Conv2D layers ($64 \to 128 \to 128$ channels), each followed by ReLU, then $2\times2$ MaxPool --- output: $128 \times 8 \times 8$.
    \item \textbf{Block~3:} Two Conv2D layers ($128 \to 256 \to 256$ channels), each followed by ReLU, then $2\times2$ MaxPool --- output: $256 \times 4 \times 4$.
    \item \textbf{Classifier:} Flatten $\to$ fully connected (FC) layer ($4096 \to 512$) $\to$ ReLU $\to$ Dropout(0.5)~\cite{srivastava_dropout_2014} $\to$ FC layer ($512 \to C$).
\end{itemize}
All weights are initialized with Xavier uniform initialization~\cite{glorot_understanding_2010}. The implementation
exposes three cut-layer options:
\begin{itemize}
  \item \textbf{cut~1} (default): after Block~1 MaxPool; activation dimension $64 \times 16 \times 16 = 16{,}384$.
  \item \textbf{cut~2}: after Block~2 MaxPool; activation dimension $128 \times 8 \times 8 = 8{,}192$.
  \item \textbf{cut~3}: after the first FC layer (pre-Dropout); activation dimension $512$.
\end{itemize}

\paragraph{Linearity condition for exact additive-share evaluation.} As established in Section~\ref{sec:pre}, the secure protocol's gradient decomposition (Equation~\eqref{eq:grad_decomp}) is algebraically exact when the operations evaluated independently on additive shares are linear/affine. This condition should be read narrowly. It is a condition for the lightweight masking protocol, not a mandate that the model designer choose a linear $f_{\text{server}}$. In ordinary base mode, any differentiable server-side subnetwork can be trained without affecting accuracy losslessness, provided that the virtual batch, parameters, optimizer state, and returned cut-layer gradient match the centralized computation.

The distinction matters because secure mode hides the cut activation by keeping it split. The two servers do not reconstruct the activation before evaluating the sharewise server path; each server evaluates its own share and the partial results are combined. If that sharewise path is linear/affine, masking itself introduces no approximation or accuracy-losslessness concern. If it contains nonlinear functions and those functions are applied separately to shares, the represented function and gradient may differ from the centralized computation. A nonlinear $f_{\text{server}}$ is therefore still a valid modeling choice, but exact secure use of such a model requires additional secure nonlinear protocols; otherwise the run should be described as an approximation.

For cut~1 and cut~2, the server model includes convolutional layers with ReLU activations, so the lightweight additive-share protocol alone is not exact for those cuts. For \textbf{cut~3}, the server model reduces to a single linear projection $\mathbb{R}^{512} \to \mathbb{R}^{C}$ (the final fully connected layer without nonlinear activation), so masking introduces no loss under the exact protocol. Accordingly, this CNN uses cut~3 for exact secure evaluation, whereas base mode imposes no such restriction and supports all three cut points.

\begin{remark}[Approximation regime for cut~1--2 in secure mode]
For secure cut~1 or cut~2, ``approximation'' means that nonlinear server layers are evaluated separately on additive shares. It does not mandate linear $f_{\text{server}}$ in general or imply that server-side backpropagation causes accuracy loss. If the sharewise path is linear/affine, additive masking preserves the function and gradient; the approximation arises only when nonlinear operations such as ReLU are applied without secure nonlinear protocols. Proposition~\ref{prop:secure_equiv} then does not apply. We report secure cut~1--2 only for sensitivity; strict secure-gradient correctness in this architecture requires cut~3 or secure nonlinear layers.
\end{remark}

\subsection{Virtual Batch Construction}

Virtual batch construction, inherited from Traversal Learning~\cite{batbaatar_traversal_2025}, is the mechanism by which TL++ achieves gradient equivalence to centralized mini-batch training under non-IID data. Let each node $i$ hold $n_i$ samples indexed by a local index set $\mathcal{I}_i$. Prior to training, the orchestrator collects the per-node sample counts $\{n_i\}_{i=1}^N$ (message~A) and constructs a global index set $\mathcal{I} = \bigsqcup_{i=1}^{N} \mathcal{I}_i$ of cardinality $n$.

For each training epoch, the orchestrator randomly permutes $\mathcal{I}$ and partitions it into $M = \lceil n / B \rceil$ non-overlapping \emph{virtual batches} $\mathcal{B}_1, \ldots, \mathcal{B}_M$, each of target size $B$ except possibly the final batch. Each virtual batch $\mathcal{B}_m$ contains indices drawn from potentially all $N$ nodes, mirroring a centralized mini-batch drawn from the pooled dataset $\mathcal{D}$ under the same without-replacement epoch shuffle. When the final batch is smaller than $B$, the normalization uses its actual size $|\mathcal{B}_m|$.

For each batch $\mathcal{B}_m$, the orchestrator decomposes it by node: $\mathcal{B}_m^{(i)} = \mathcal{B}_m \cap \mathcal{I}_i$, and sends each node $i$ the sub-index set $\mathcal{B}_m^{(i)}$ (message~B). Node $i$ loads only the corresponding local samples $\{\mathbf{x}^{(i)}_j\}_{j \in \mathcal{B}_m^{(i)}}$, computes their cut-layer activations, and returns them. The orchestrator concatenates activations from all nodes to form the full virtual-batch activation matrix $\mathbf{A}_m \in \mathbb{R}^{B \times d}$, where $d$ is the cut-layer dimension.

\paragraph{Label handling.}
Algorithm~\ref{alg:base} and Section~\ref{sec:secure_proto} assume the orchestrator has the virtual-batch labels $\mathbf{Y}_m$ for loss computation. Node $i$ may attach labels $\{y^{(i)}_j\}_{j\in\mathcal{B}_m^{(i)}}$ to its activation message or send them over the same authenticated channel. This does not centralize raw inputs, but it is an explicit label-disclosure assumption: the orchestrator may observe labels, node membership, and batch participation unless additional anonymization or batching is used. These metadata are outside the activation-sharing guarantee. Section~\ref{sec:label_u_shape} discusses U-shaped variants that keep labels local~\cite{vepakomma_split_2018}.

\begin{proposition}[Gradient Equivalence]
    \label{prop:equiv}
    Let $\mathbf{A}_m$ be the virtual-batch activation matrix for batch $\mathcal{B}_m$ with the random shuffling described above. Conditional on the same model parameters, optimizer state, and batch indices, the base TL++ server-side gradient computed from $\mathbf{A}_m$ is identical to the gradient that centralized mini-batch SGD would compute on $\mathcal{B}_m$. Over the epoch shuffle, this gradient has the same sampling distribution as centralized mini-batch SGD on the pooled dataset $\mathcal{D}$.
\end{proposition}

\begin{proof}[Proof sketch]
    The virtual batch index construction samples the same index sets that a centralized without-replacement epoch shuffle would produce. Since the mapping from indices to samples is deterministic, the batch $\{\mathbf{x}_j, y_j\}_{j \in \mathcal{B}_m}$ is identical to the corresponding centralized mini-batch. The node models compute the same cut-layer activations that the centralized model would compute for those samples, and the orchestrator concatenates them in batch order before applying the same server model and loss. Therefore the subsequent forward and backward passes are algebraically identical to centralized SGD on the same batch; no approximation is introduced in base mode.
\end{proof}

In base mode, the concatenated activation matrix $\mathbf{A}_m$ is formed in plaintext at the orchestrator. In secure mode, the orchestrator holds the share matrix $\mathbf{A}^{(1)}_m$ and the helper holds $\mathbf{A}^{(2)}_m$, with $\mathbf{A}^{(1)}_m + \mathbf{A}^{(2)}_m = \mathbf{A}_m$; concatenation and aggregation proceed share-wise by the linearity of additive secret sharing (Equations~\eqref{eq:lin_add}--\eqref{eq:lin_scale}).

\subsection{Standard (Base) Protocol}

Algorithm~\ref{alg:base} summarizes the base-mode training protocol for a single virtual batch $\mathcal{B}_m$.

\begin{algorithm}[t]
    \caption{TL++ Base Mode --- Single Batch $\mathcal{B}_m$}
    \label{alg:base}
    \begin{algorithmic}[1]
    \Require Virtual batch indices $\mathcal{B}_m$; current parameters $\boldsymbol{\theta}_{\text{node}}, \boldsymbol{\theta}_{\text{server}}$; learning rate $\eta$; sub-batch sizes $b_i = |\mathcal{B}_m^{(i)}|$.
    \For{each node $i$ in parallel}
      \State Receive sub-indices $\mathcal{B}_m^{(i)}$ from orchestrator.
      \State Compute activations: $\mathbf{A}^{(i)} \leftarrow f_{\text{node}}\!\left(\{\mathbf{x}^{(i)}_j\}_{j \in \mathcal{B}_m^{(i)}}; \boldsymbol{\theta}_{\text{node}}\right)$.
      \State Send $\mathbf{A}^{(i)}$ to orchestrator.
    \EndFor
    \State \textbf{Orchestrator:} Concatenate $\mathbf{A}_m \leftarrow [\mathbf{A}^{(1)} \,\|\, \cdots \,\|\, \mathbf{A}^{(N)}]$.
    \State \hspace{2.3em} Forward pass: $\hat{\mathbf{Y}} \leftarrow f_{\text{server}}(\mathbf{A}_m;\,\boldsymbol{\theta}_{\text{server}})$.
    \State \hspace{2.3em} Compute loss: $L \leftarrow \ell(\hat{\mathbf{Y}},\,\mathbf{Y}_m)$.
    \State \hspace{2.3em} Backward pass: $\mathbf{G}_m \leftarrow \partial L / \partial \mathbf{A}_m$; \; $\Delta\boldsymbol{\theta}_{\text{server}} \leftarrow \partial L / \partial \boldsymbol{\theta}_{\text{server}}$.
    \State \hspace{2.3em} Update: $\boldsymbol{\theta}_{\text{server}} \leftarrow \boldsymbol{\theta}_{\text{server}} - \eta\,\Delta\boldsymbol{\theta}_{\text{server}}$.
    \For{each node $i$ in parallel}
      \State Receive sub-gradient $\mathbf{G}^{(i)} \leftarrow \mathbf{G}_m[\mathcal{B}_m^{(i)},:]$ from orchestrator.
      \State Backpropagate: $\Delta\boldsymbol{\theta}^{(i)}_{\text{node}} \leftarrow \partial L / \partial \boldsymbol{\theta}_{\text{node}}$ via $\mathbf{G}^{(i)}$.
      \State Send $\Delta\boldsymbol{\theta}^{(i)}_{\text{node}}$ to orchestrator.
    \EndFor
    \State \textbf{Orchestrator:} Aggregate: $\Delta\boldsymbol{\theta}_{\text{node}} \leftarrow \frac{1}{B}\sum_{i=1}^{N} b_i \cdot \Delta\boldsymbol{\theta}^{(i)}_{\text{node}}$.
    \State \hspace{2.3em} Update: $\boldsymbol{\theta}_{\text{node}} \leftarrow \boldsymbol{\theta}_{\text{node}} - \eta\,\Delta\boldsymbol{\theta}_{\text{node}}$.
    \State \hspace{2.3em} Broadcast updated $\boldsymbol{\theta}_{\text{node}}$ to all nodes.
    \end{algorithmic}
\end{algorithm}

The aggregation in line~17 weights each node's gradient contribution by its sub-batch size $b_i = |\mathcal{B}_m^{(i)}|$, yielding the correct gradient of the global loss on the virtual batch regardless of how the $B$ samples are distributed across nodes. When $b_i = B/N$ for all $i$ (equal sub-batch sizes), this reduces to the simple arithmetic mean $\frac{1}{N}\sum_i \Delta\boldsymbol{\theta}^{(i)}_{\text{node}}$.

The base protocol is mathematically equivalent to centralized mini-batch SGD on $\mathcal{B}_m$ (Proposition~\ref{prop:equiv}). The per-batch communication cost is $O(B \cdot d + N \cdot |\boldsymbol{\theta}_{\text{node}}|)$, dominated by activation transmission and node-model gradient exchange, which is substantially less than FL's per-round cost of $O(N \cdot |\boldsymbol{\theta}|)$ when $B \cdot d \ll N \cdot |\boldsymbol{\theta}|$.

\subsection{Secure Protocol}
\label{sec:secure_proto}

The secure protocol adds additive-secret-sharing MPC via a non-colluding helper~\cite{damgard_multiparty_2012,mohassel_secureml_2017,mohassel_aby3_2018,wagh_securenn_2019}. It has four phases per virtual batch. The exact lightweight path below assumes that operations applied independently to the two shares are linear or affine (Section~\ref{sec:model}); nonlinear server networks require secure nonlinear subprotocols to make the decompositions exact.

\paragraph{Phase 1: Share Generation.} Each node $i$ computes its cut-layer activation $\mathbf{A}^{(i)}$ as in Equation~\eqref{eq:activation}. It then samples a random mask $\mathbf{R}^{(i)} \sim \mathcal{U}(\mathbb{R}^{|\mathcal{B}_m^{(i)}| \times d})$ and computes two shares:
\begin{equation}
  \mathbf{A}^{(i)}_1 = \mathbf{R}^{(i)}, \qquad \mathbf{A}^{(i)}_2 = \mathbf{A}^{(i)} - \mathbf{R}^{(i)}.
  \label{eq:node_share}
\end{equation}
Node $i$ transmits $\mathbf{A}^{(i)}_1$ to the orchestrator and $\mathbf{A}^{(i)}_2$ to the helper server, retaining neither share locally after transmission.

\paragraph{Phase 2: Independent Forward Passes and Label-Based Loss.} Both servers concatenate their respective share matrices:
\begin{equation}
  \mathbf{A}^{(k)}_m = \bigl[\mathbf{A}^{(1)}_k \,\|\, \cdots \,\|\, \mathbf{A}^{(N)}_k\bigr], \quad k \in \{1, 2\}.
\end{equation}
Each server independently executes the linear/affine sharewise server path on its share:
\begin{equation}
  \hat{\mathbf{Y}}_k = f_{\text{server}}\!\left(\mathbf{A}^{(k)}_m;\, \boldsymbol{\theta}_{\text{server}}\right), \quad k \in \{1, 2\}.
\end{equation}
The helper sends $\hat{\mathbf{Y}}_2$ to the orchestrator over the dedicated inter-server coordination channel. The orchestrator combines it with $\hat{\mathbf{Y}}_1$ and reconstructs
\begin{equation}
  \hat{\mathbf{Y}} = \hat{\mathbf{Y}}_1 + \hat{\mathbf{Y}}_2.
\end{equation}
Because this sharewise server path is linear/affine, with any bias handled once or shared consistently, $\hat{\mathbf{Y}}_1 + \hat{\mathbf{Y}}_2 = f_{\text{server}}(\mathbf{A}^{(1)}_m + \mathbf{A}^{(2)}_m) = f_{\text{server}}(\mathbf{A}_m)$ holds exactly. The orchestrator uses the label vector $\mathbf{Y}_m$ associated with the virtual batch to compute the scalar loss $L = \ell(\hat{\mathbf{Y}}, \mathbf{Y}_m)$ and the output-gradient signal
\begin{equation}
  \mathbf{H}_m = \frac{\partial L}{\partial \hat{\mathbf{Y}}}.
  \label{eq:output_grad}
\end{equation}
The orchestrator sends $\mathbf{H}_m$ to the helper for backpropagation. This is the only point where output predictions are reconstructed, adding one inter-server round independent of $N$. Labels are visible to the orchestrator but need not be sent to the helper; $\mathbf{H}_m$ is nevertheless label-dependent and should be treated as a training signal (Section~\ref{sec:label_u_shape}).

\paragraph{Phase 3: Secure Backward Pass.} Given $\mathbf{H}_m$, each server $k$ computes the partial cut-layer gradient associated with its activation share:
\begin{equation}
  \mathbf{G}^{(k)}_m
  =
  \mathbf{H}_m \cdot
  \frac{\partial \hat{\mathbf{Y}}_k}{\partial \mathbf{A}^{(k)}_m},
  \quad k \in \{1,2\}.
\end{equation}
For this linear/affine sharewise server path, the full cut-layer gradient decomposes additively:
\begin{equation}
  \mathbf{G}_m = \frac{\partial L}{\partial \mathbf{A}_m}
  =
  \mathbf{H}_m \cdot \frac{\partial \hat{\mathbf{Y}}_1}{\partial \mathbf{A}^{(1)}_m}
  +
  \mathbf{H}_m \cdot \frac{\partial \hat{\mathbf{Y}}_2}{\partial \mathbf{A}^{(2)}_m}
  = \mathbf{G}^{(1)}_m + \mathbf{G}^{(2)}_m.
  \label{eq:grad_decomp}
\end{equation}
Thus $(\mathbf{G}^{(1)}_m,\mathbf{G}^{(2)}_m)$ is a valid additive sharing of $\mathbf{G}_m$: neither server sees the plaintext cut-layer gradient. Server $k$ returns sub-gradient share $\mathbf{G}^{(i)}_k=\mathbf{G}^{(k)}_m[\mathcal{B}_m^{(i)},:]$ to node $i$.

\paragraph{Phase 4: Node-Side Update and Aggregation.} Node $i$ receives $\mathbf{G}^{(i)}_1$ and $\mathbf{G}^{(i)}_2$, reconstructs $\mathbf{G}^{(i)}=\mathbf{G}^{(i)}_1+\mathbf{G}^{(i)}_2$, and backpropagates through $f_{\text{node}}$ to compute $\Delta\boldsymbol{\theta}^{(i)}_{\text{node}}$. The orchestrator aggregates node gradients using the base-mode size-weighted formula.

\begin{proposition}[Secure Gradient Correctness]
    \label{prop:secure_equiv}
    Let Proposition~\ref{prop:equiv} hold, and suppose the operations evaluated sharewise by the two servers are linear or affine, with biases handled once or shared consistently. The gradient update produced by the secure protocol is numerically identical to that of the base protocol on the same virtual batch $\mathcal{B}_m$, and therefore to centralized mini-batch SGD on $\mathcal{B}_m$. Thus, delegating the server-side portion of backpropagation does not affect accuracy losslessness when the returned cut-layer gradient is exact; any loss of exactness comes from approximating the server computation under masking, not from the cut itself.
\end{proposition}

\begin{proof}
    By construction, $\mathbf{A}^{(1)}_m + \mathbf{A}^{(2)}_m = \mathbf{A}_m$ exactly (Equation~\eqref{eq:node_share}). Because the sharewise server computation is linear/affine, $\hat{\mathbf{Y}}_1 + \hat{\mathbf{Y}}_2 = f_{\text{server}}(\mathbf{A}_m)$ exactly. Its Jacobian is independent of the hidden plaintext activation, so the gradient decomposition in Equation~\eqref{eq:grad_decomp} is also algebraically exact. No approximation error is introduced by masking in this case, and the secure protocol is numerically equivalent to the base protocol.
\end{proof}

\subsection{Optimization}

\paragraph{Learning rate and batch size.} TL++ follows the standard linear scaling rule~\cite{goyal_accurate_2017}: when the virtual batch size $B$ is increased by a factor $k$, the learning rate $\eta$ is also scaled by $k$ to maintain comparable gradient noise. This rule applies identically in base and secure modes, since both compute the same gradient estimates.

\paragraph{Noise augmentation for privacy tuning.} To provide a configurable privacy-utility trade-off, TL++ optionally adds independent calibrated Gaussian noise to activation shares and gradient shares before transmission, following the standard DP intuition of calibrated noise and formal accounting~\cite{dwork_differential_2006,abadi_deep_2016,mironov_renyi_2017}. For activation shares:
\begin{equation}
  \tilde{\mathbf{A}}^{(i)}_k = \mathbf{A}^{(i)}_k + \boldsymbol{\varepsilon}^{(i)}_{\text{act},k}, \quad \boldsymbol{\varepsilon}^{(i)}_{\text{act},k} \sim \mathcal{N}(\mathbf{0},\,\sigma_{\text{act}}^2\mathbf{I}),
  \label{eq:noise_act}
\end{equation}
and for cut-layer gradient shares returned to nodes:
\begin{equation}
  \tilde{\mathbf{G}}^{(i)}_k = \mathbf{G}^{(i)}_k + \boldsymbol{\varepsilon}^{(i)}_{\text{grad},k}, \quad \boldsymbol{\varepsilon}^{(i)}_{\text{grad},k} \sim \mathcal{N}(\mathbf{0},\,\sigma_{\text{grad}}^2\mathbf{I}).
  \label{eq:noise_grad}
\end{equation}
With $\sigma_{\text{act}}=\sigma_{\text{grad}}=0$, the protocol is exact additive sharing with no noise-induced accuracy loss. Nonzero noise can empirically obfuscate intermediate values against reconstruction attacks~\cite{zhu_deep_2019,zhao_idlg_2020,geiping_inverting_2020}, but changes the optimization path and may reduce accuracy. The parameters are independent because gradient shares can reveal more than activation shares; default values are $\sigma_{\text{act}}=0.02$ and $\sigma_{\text{grad}}=0.10$. These defaults are not differential-privacy claims; formal accounting is deferred to Section~\ref{sec:discussion}.

\paragraph{LoRA integration.}
For large language model (LLM) adaptation, TL++ can train only LoRA-adapted~\cite{hu_lora_2021} modules, following adapters, adaptive low-rank allocation, and quantized LoRA~\cite{houlsby_parameter-efficient_2019,zhang_adalora_2023,dettmers_qlora_2023}. Base matrices $\mathbf{W}\in\mathbb{R}^{m\times n}$ are frozen and only $\Delta\mathbf{W}=\mathbf{B}\mathbf{A}$, with $\mathbf{A}\in\mathbb{R}^{r\times n}$, $\mathbf{B}\in\mathbb{R}^{m\times r}$, and $r\ll\min(m,n)$, is trained. LoRA lowers synchronization cost, but it does not make an entire transformer server path compatible with sharewise additive evaluation: attention softmax, layer normalization, and other nonlinear components still require placement below the cut, base-mode plaintext evaluation, or secure nonlinear protocols. The exact sharing argument applies only to linear portions.


\section{Experiments}
\label{sec:experiments}

This section evaluates whether TL++ retains centralized-training utility while reducing full-model exchange. We report CIFAR-10~\cite{krizhevsky_learning_2009} and BioGPT/\allowbreak{}PubMedQA~\cite{luo_biogpt_2022,jin_pubmedqa_2019} results with deterministic tensor-payload accounting rather than packet-level runtime measurements. Because cut depth, mode, and synchronization cadence affect the utility--communication trade-off, we state the conditions under which each operating point is valid.

\subsection{Experimental Setup and Reporting Protocol}
\label{sec:setup}

\paragraph{Workloads and configurations.} The image-classification experiments use CIFAR-10~\cite{krizhevsky_learning_2009}. We report class-count tasks with $C \in \{2,3,4,5,6,7,8,9,10\}$, where $C=10$ is the full task. The language-model experiments use BioGPT~\cite{luo_biogpt_2022} on PubMedQA~\cite{jin_pubmedqa_2019} with both full fine-tuning and LoRA adaptation~\cite{hu_lora_2021}. For CIFAR-10, we compare centralized training, TL++ base mode, TL++ secure mode, and representative distributed-learning baselines. Base TL++ is evaluated at cut~1, cut~2, and cut~3. Secure TL++ is also reported at all three cuts, but only cut~3 satisfies the exact additive-share evaluation condition used in Proposition~\ref{prop:secure_equiv}. Secure cut~1 and secure cut~2 place nonlinear operations above the cut; without an added secure nonlinear protocol, they are therefore reported as sensitivity results rather than exact secure configurations.

\paragraph{Metrics and interpretation.} Accuracy is test accuracy percentage, reported as mean $\pm$ standard deviation over five independent runs. Comparisons are descriptive because per-seed paired predictions and trajectories are unavailable; TL++ means that exceed centralized baselines are treated as run-to-run variation. The accuracy tables report utility metrics only; Section~\ref{sec:comm_exp} reports communication costs with tensor payloads plus serialization framing, headers, one-way latency, and protocol overhead. Accordingly, we describe byte-level ``payload reduction'' rather than wall-clock speedup unless runtime measurements are available.

\subsection{CIFAR-10 Accuracy Across Class Counts}
\label{sec:acc_exp}

Table~\ref{tab:accuracy} shows that TL++ remains close to centralized training across the CIFAR-10 class-count sweep. On the full task, centralized training reaches $92.03\pm0.15\%$, base cut~1 reaches $91.41\pm0.19\%$ (0.62 points lower), and exact secure cut~3 reaches $90.93\pm0.17\%$ (1.10 points lower), preserving most centralized utility while keeping raw examples distributed.

\begin{table*}[t]
    \caption{CIFAR-10 test accuracy (\%) across class-count tasks. Each entry is mean $\pm$ standard deviation over five runs. The best mean in each row is shown in \textbf{bold}; ties and small differences should be interpreted descriptively because per-seed paired tests were not available.}
    \label{tab:accuracy}
    \centering
    \footnotesize
    \begin{tabular}{@{}rccccccc@{}}
        \toprule
        Classes $C$ & Centralized & base cut~1 & base cut~2 & base cut~3 & secure cut~1 & secure cut~2 & secure cut~3 \\
        \midrule
        2  & $99.04 \pm 0.17$ & $98.98 \pm 0.09$ & $98.84 \pm 0.14$ & $\mathbf{99.09 \pm 0.13}$ & $98.75 \pm 0.12$ & $98.77 \pm 0.10$ & $98.93 \pm 0.13$ \\
        3  & $\mathbf{97.37 \pm 0.15}$ & $96.99 \pm 0.48$ & $96.56 \pm 0.42$ & $96.84 \pm 0.35$ & $95.29 \pm 0.68$ & $95.92 \pm 0.51$ & $96.96 \pm 0.33$ \\
        4  & $95.06 \pm 0.55$ & $\mathbf{95.19 \pm 0.11}$ & $94.73 \pm 0.51$ & $94.24 \pm 0.85$ & $93.27 \pm 1.04$ & $93.76 \pm 0.54$ & $94.78 \pm 0.70$ \\
        5  & $\mathbf{94.56 \pm 0.59}$ & $94.34 \pm 0.48$ & $93.82 \pm 0.62$ & $92.84 \pm 0.75$ & $91.74 \pm 0.55$ & $92.95 \pm 0.86$ & $94.34 \pm 0.21$ \\
        6  & $91.57 \pm 0.71$ & $\mathbf{91.90 \pm 0.26}$ & $91.48 \pm 0.29$ & $90.82 \pm 0.59$ & $89.09 \pm 1.07$ & $90.88 \pm 0.53$ & $91.62 \pm 0.49$ \\
        7  & $\mathbf{92.30 \pm 0.13}$ & $91.94 \pm 0.20$ & $91.59 \pm 0.18$ & $90.35 \pm 0.62$ & $89.95 \pm 1.30$ & $91.14 \pm 0.32$ & $91.29 \pm 0.84$ \\
        8  & $\mathbf{91.97 \pm 0.89}$ & $91.57 \pm 0.58$ & $91.03 \pm 0.24$ & $88.91 \pm 0.63$ & $90.33 \pm 0.96$ & $90.55 \pm 0.33$ & $91.52 \pm 0.18$ \\
        9  & $\mathbf{92.39 \pm 0.17}$ & $91.88 \pm 0.17$ & $91.09 \pm 0.17$ & $90.02 \pm 0.38$ & $90.47 \pm 0.90$ & $90.87 \pm 0.34$ & $91.38 \pm 0.24$ \\
        10 & $\mathbf{92.03 \pm 0.15}$ & $91.41 \pm 0.19$ & $90.65 \pm 0.25$ & $89.67 \pm 0.27$ & $90.36 \pm 0.69$ & $90.16 \pm 1.07$ & $90.93 \pm 0.17$ \\
        \bottomrule
    \end{tabular}
\end{table*}

Averaged over $C=2$ through $C=10$, centralized training obtains 94.03\%, base cut~1 obtains 93.80\%, and secure cut~3 obtains 93.53\%. The corresponding average gaps from centralized training are 0.23 and 0.50 percentage points. Accuracy declines as the label space becomes harder, but TL++ follows the same trend as centralized training rather than showing a separate collapse from virtual-batch construction.

The sweep also shows why the best trusted and exact secure cuts differ. As classes increase, cut placement matters more: earlier cuts leave more representation learning on the server and preserve base-mode accuracy, whereas the deepest cut reduces the server side to the linear portion required for exact sharing. Base cut~1 therefore optimizes utility and payload in trusted mode, while secure cut~3 optimizes algebraic compatibility.

\subsection{Aggregate CIFAR-10 Comparison and Gap Analysis}
\label{sec:gap_analysis}

Table~\ref{tab:gap_summary} summarizes the CIFAR-10 method comparison using full ten-class accuracy, average accuracy across the class-count sweep, and gaps relative to centralized training.

\begin{table}[t]
    \caption{Aggregate CIFAR-10 accuracy and gap summary. Average accuracy and gaps are computed over $C=2$ through $C=10$ for measured rows. ``Worst gap'' is the most negative class-count gap relative to centralized training.}
    \label{tab:gap_summary}
    \centering
    \small
    \begin{tabular}{@{}lcccc@{}}
        \toprule
        Configuration & 10-class & Avg. & Avg. gap & Worst gap \\
        \midrule
        Centralized & 92.03 & 94.03 & -- & -- \\
        \midrule
        FedAvg~\cite{mcmahan_communication-efficient_2017} & 74.56 & 79.48 & -14.55 & -18.45 \\
        FedProx~\cite{li_fedprox_2020} & 74.64 & 79.45 & -14.58 & -18.45 \\
        SCAFFOLD~\cite{karimireddy_scaffold_2020} & 71.32 & 76.12 & -17.92 & -22.58 \\
        Standard SL~\cite{vepakomma_split_2018} & 78.88 & 78.84 & -15.20 & -27.50 \\
        SplitFed~\cite{thapa_splitfed_2022} & 74.62 & 79.46 & -14.58 & -18.53 \\
        \midrule
        TL++ base, cut~1 & 91.41 & 93.80 & -0.23 & -0.62 \\
        TL++ base, cut~2 & 90.65 & 93.31 & -0.72 & -1.30 \\
        TL++ base, cut~3 & 89.67 & 92.53 & -1.50 & -3.06 \\
        TL++ secure, cut~1 & 90.36 & 92.14 & -1.89 & -2.82 \\
        TL++ secure, cut~2 & 90.16 & 92.78 & -1.25 & -1.87 \\
        TL++ secure, cut~3 & 90.93 & 93.53 & -0.50 & -1.10 \\
        \bottomrule
    \end{tabular}
\end{table}

The aggregate results distinguish TL++ from the measured FL-style and split-learning baselines in this workload summary. The baseline methods are substantially below centralized training, whereas TL++ base and secure configurations remain within a few percentage points. Among base-mode runs, cut~1 has the best full-task accuracy and the smallest average gap. Among secure-mode runs, cut~3 has the best full-task accuracy and the smallest average gap, and it is also the cut that satisfies the exact linear-server condition. This distinction is the basis for the deployment guidance in Section~\ref{sec:discussion}: base cut~1 is the strongest trusted configuration in the reported CIFAR-10 experiments, while secure cut~3 is the supported exact secure configuration.

The comparison is most informative when interpreted through the optimization objective of each method. FL-style baselines optimize local objectives between synchronization events; under label skew, the local gradients may point away from the pooled-data gradient and therefore require additional correction, more rounds, or stronger regularization. Standard SL avoids full-model exchange, but it does not by itself create a pooled virtual batch across participants. TL++ directly targets this missing ingredient: it keeps raw samples local while forming a virtual batch that better approximates centralized gradient aggregation. The gap summary therefore supports the claim that virtual-batch construction is the primary source of TL++'s utility advantage in this setting, while cut-layer transmission is the primary source of its communication advantage.

At the same time, the table should not be read as a universal ranking of all distributed learning systems. The relative baseline performance depends on the non-IID partition, model capacity, optimizer schedule, number of local steps, and communication budget. The contribution of Table~\ref{tab:gap_summary} is more specific: under the reported CIFAR-10 protocol, TL++ occupies a near-centralized utility regime that the selected FL-style and split-learning baselines do not reach.

\subsection{Cut-Layer Sensitivity}
\label{sec:ablations}

Table~\ref{tab:cut_ablation} isolates the effect of cut placement. The cut dimension is included because it affects both accuracy and communication: earlier cuts preserve more server-side computation, while deeper cuts reduce activation and share sizes.

\begin{table}[t]
    \caption{Cut-layer sensitivity on CIFAR-10. ``10-class'' reports the full CIFAR-10 task. ``Avg.'' averages mean accuracies across $C=2$ through $C=10$. $\Delta$ Avg. reports Secure Avg. minus Base Avg. at the same cut.}
    \label{tab:cut_ablation}
    \centering
    \small
    \begin{tabular}{@{}lrrrrrr@{}}
        \toprule
        Cut & Dim. & Base 10-class & Secure 10-class & Base Avg. & Secure Avg. & $\Delta$ Avg. \\
        \midrule
        cut~1 & $16{,}384$ & 91.41 & 90.36 & 93.80 & 92.14 & -1.66 \\
        cut~2 & $8{,}192$  & 90.65 & 90.16 & 93.31 & 92.78 & -0.53 \\
        cut~3 & $512$     & 89.67 & 90.93 & 92.53 & 93.53 & +1.00 \\
        \bottomrule
    \end{tabular}
\end{table}

Base-mode accuracy decreases with depth, from 91.41\% at cut~1 to 90.65\% at cut~2 and 89.67\% at cut~3, suggesting that this architecture benefits from leaving more representation learning on the server. Secure mode differs: secure cut~3 is strongest, with 90.93\% ten-class and 93.53\% average accuracy. We do not treat its advantage over base cut~3 as a general benefit of secrecy; rather, exact additive sharing at the linear cut introduced no observable penalty.

The ablation also shows that activation dimension alone is insufficient. Deeper cuts shrink the cut tensor from $16{,}384$ dimensions to $512$, but move trainable parameters to the node side; with per-batch synchronization, total payload need not decrease monotonically. For secure TL++, cut~1 and cut~2 are sensitivity runs before nonlinear server layers and do not satisfy the proof. Secure cut~3 is therefore the correctness point even if its conservative-schedule payload is not minimal; communication should then be improved through amortization, compression, or parameter-efficient updates.

\subsection{Performance as Label-Space Complexity Increases}
\label{sec:class_complexity}

The class-count sweep provides a stress test for heterogeneous classification. The average over the 2-, 3-, and 4-class tasks is 97.16\% for centralized training, 97.05\% for base cut~1, and 96.89\% for secure cut~3. The average over the harder 8-, 9-, and 10-class tasks is 92.13\%, 91.62\%, and 91.28\%, respectively. The ordering is stable: centralized training is the reference, base cut~1 is the strongest base TL++ configuration, and secure cut~3 is the strongest exact secure TL++ configuration. This supports the bounded claim that TL++ tracks the centralized accuracy curve closely as class complexity increases.

The moderate widening on harder tasks is expected because the model separates more decision regions under the same resolution and distributed constraints. The gap grows smoothly rather than abruptly; a failed virtual-batch mechanism would likely degrade more sharply as classes are added. Thus, harder tasks mainly reflect CIFAR-10 difficulty while TL++ remains close to the centralized trajectory, consistent with virtual batches recovering centralized-style gradient aggregation rather than making classification easier.

\subsection{BioGPT Fine-Tuning and LoRA Adaptation}
\label{sec:lora_exp}

Table~\ref{tab:lora} evaluates whether the TL++ pattern extends beyond convolutional image classification. The BioGPT/\allowbreak{}PubMedQA results~\cite{luo_biogpt_2022,jin_pubmedqa_2019} compare centralized training, distributed-learning baselines, TL++ base, and TL++ secure under both full fine-tuning and parameter-efficient LoRA adaptation~\cite{hu_lora_2021,houlsby_parameter-efficient_2019,dettmers_qlora_2023}.

\begin{table}[t]
    \caption{BioGPT/\allowbreak{}PubMedQA accuracy (\%) for full fine-tuning and LoRA adaptation. Entries are mean $\pm$ standard deviation over five runs.}
    \label{tab:lora}
    \centering
    \small
    \begin{tabular}{@{}lcc@{}}
        \toprule
        Method & Full fine-tuning & LoRA \\
        \midrule
        Centralized & $81.33 \pm 0.47$ & $81.07 \pm 1.25$ \\
        \midrule
        FedAvg~\cite{mcmahan_communication-efficient_2017} & $65.67 \pm 3.21$ & $55.89 \pm 2.22$ \\
        FedProx~\cite{li_fedprox_2020} & $65.33 \pm 2.73$ & $59.44 \pm 5.17$ \\
        SCAFFOLD~\cite{karimireddy_scaffold_2020} & $51.67 \pm 0.33$ & $53.67 \pm 1.33$ \\
        Standard SL~\cite{vepakomma_split_2018} & $55.44 \pm 5.10$ & $67.11 \pm 2.71$ \\
        SplitFed~\cite{thapa_splitfed_2022} & $67.44 \pm 3.50$ & $61.44 \pm 1.26$ \\
        \midrule
        TL++ base & $80.87 \pm 1.68$ & $81.20 \pm 2.13$ \\
        TL++ secure & $\mathbf{83.67 \pm 0.75}$ & $\mathbf{82.60 \pm 2.24}$ \\
        \bottomrule
    \end{tabular}
\end{table}

TL++ remains close to centralized fine-tuning on PubMedQA. For full fine-tuning, TL++ base reaches 80.87\%, within 0.46 percentage points of centralized training, while TL++ secure reaches 83.67\%. For LoRA, TL++ base reaches 81.20\%, comparable to centralized LoRA at 81.07\%, and TL++ secure reaches 82.60\%. Because the standard deviations are non-negligible and the aggregate summaries do not provide paired trajectories, these results should be read as evidence of comparable utility rather than as a definitive ranking of TL++ base, TL++ secure, full fine-tuning, and LoRA.

The language-model experiment suggests that traversal learning is not tied to CNNs: utility can be preserved when trainable state is limited to parameter-efficient modules. LoRA reduces trainable parameters~\cite{hu_lora_2021}, consistent with adapter and quantized fine-tuning work~\cite{houlsby_parameter-efficient_2019,dettmers_qlora_2023}, and therefore reduces synchronization pressure. The security interpretation is narrower: unless nonlinear transformer operations are below the cut or protected by secure nonlinear protocols, BioGPT secure results are protocol-level sensitivity evidence rather than a proof of exact secure transformer training. PubMedQA is also a small case study; longer-context tasks, generative metrics, calibration, and domain-shift tests are needed before broad claims about TL++ for medical language models.

\subsection{Communication-Efficiency Measurements}
\label{sec:comm_exp}

We report analytic per-virtual-batch communication and computation for FedAvg, FedProx, SCAFFOLD, standard SL, SplitFed, and TL++ base/secure cut~1/2/3. The default scenario uses $B=128$, $N=10$, 32-bit tensors, 100~GFLOPs/s nodes, a 1{,}000~GFLOPs/s server, 100/200~Mb/s node uplink/downlink, 10~ms one-way node--server latency, 10~Gb/s inter-server bandwidth, 1~ms inter-server latency, 5\% serialization/framing overhead, and 512-byte headers. Accounting includes broadcasts, activations/gradients, labels and indices, node-gradient uploads, optimizer/aggregation work, FedProx/SCAFFOLD control terms, secure shares, output-share exchange, helper refresh, and helper-to-orchestrator relay. Secure cut~1--2 are approximate sensitivity measurements; cut~3 is exact.

\begin{figure}[t]
    \centering    
      \includegraphics[width=0.93\linewidth]{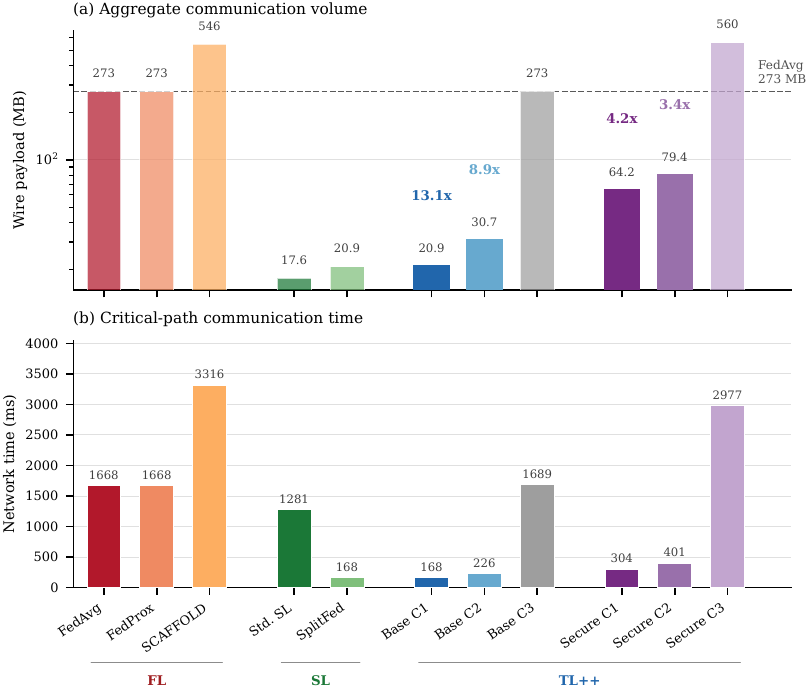}%
    \caption{Communication cost for all eleven methods under the edge--cloud link model. \textbf{(a)} Aggregate wire payload (MB, log scale), including tensor payload plus serialization/framing overhead. Bold ratio labels give the reduction relative to FedAvg for communication-efficient TL++ configurations. \textbf{(b)} Critical-path network time (ms), accounting for parallel node messages, sequential standard-SL client execution, inter-server secure-mode messages, and the current helper model refresh in secure TL++.}
    \label{fig:efficiency}
    \Description{Two-panel communication figure. Panel a is a log-scale wire-payload bar chart for FL, SL, and TL++ methods. FedAvg and FedProx are about 273 MB, SCAFFOLD about 546 MB, Base C1 about 21 MB, Base C2 about 31 MB, Secure C1 about 64 MB, Secure C2 about 79 MB, and Secure C3 about 560 MB. Panel b shows critical-path network time: Base C1 and SplitFed are lowest near 168 ms, Base C2 is near 226 ms, Secure C1 near 304 ms, Secure C2 near 401 ms, FedAvg near 1668 ms, SCAFFOLD near 3316 ms, and standard SL near 1281 ms because clients run sequentially.}
\end{figure}

Figure~\ref{fig:efficiency}(a) separates three regimes. In the full-model regime, FedAvg/FedProx require 273~MB and SCAFFOLD 546~MB; TL++ base cut~3 also reaches 273~MB because node-model synchronization dominates. In the cut-layer regime, standard SL is 17.6~MB, SplitFed and TL++ base cut~1 are 20.9~MB, and TL++ base cut~2 is 30.7~MB, giving 13.1$\times$ and 8.9$\times$ reductions for base cut~1/2 relative to FedAvg. In the secure-synchronization regime, secure cut~1/2 remain below FedAvg at 64.2/79.4~MB (4.2$\times$/3.4$\times$ reductions) but exceed activation-only costs because gradient shares traverse both server paths and helper refresh is included. Exact secure cut~3 reaches 560~MB because node-gradient shares and helper relay dominate.

Figure~\ref{fig:efficiency}(b) shows that payload alone is insufficient. Under the link model, base cut~1/2 take 168/226~ms and secure cut~1/2 take 304/401~ms, all below FedAvg/FedProx (1{,}668~ms) and SCAFFOLD (3{,}316~ms). Standard SL has the smallest payload but takes 1{,}281~ms because clients are serialized. Secure cut~3 is exact but takes 2{,}977~ms under per-batch synchronization, dominated by node-gradient-share exchange and helper relay.

\begin{figure}[t]
    \centering    
      \includegraphics[width=0.93\linewidth]{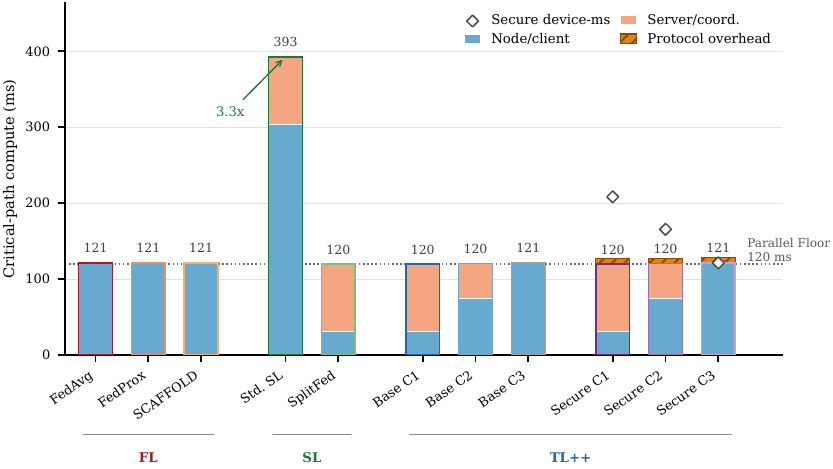}%
    {%
    }
    \caption{Critical-path computation cost for all eleven methods. Bars are stacked: node/client (blue), server/coordinator (salmon), and protocol overhead (hatched orange). The model charges forward/backward FLOPs including pointwise layers, optimizer updates, server aggregation, FedProx and SCAFFOLD per-parameter control terms, and secure sharing/reconstruction arithmetic. Protocol-overhead bars for Secure C1/C2/C3 are rendered at a minimum display height for legibility; actual values are $\leq$0.20~ms. Hollow diamonds mark total secure device-ms (critical-path time plus the helper's parallel server computation); this additional resource cost does not lengthen the critical path.}
    \label{fig:computation}
    \Description{Stacked bar chart of critical-path computation time. All parallel methods cluster near the 120 ms dotted floor. Standard SL is 393 ms, 3.3 times the parallel floor, annotated with an arrow. Secure C1, C2, and C3 show small hatched orange protocol-overhead caps at the top of their bars. Hollow diamond markers above Secure C1 (approx. 210 ms) and Secure C2 (approx. 165 ms) indicate total secure device-ms including the helper's parallel computation.}
\end{figure}

Figure~\ref{fig:computation} shows that computation is far less sensitive to the communication protocol than the network path. All parallel methods cluster near the 120~ms floor: TL++ base and secure cut~1 reach 120~ms, cut~2 reaches 120~ms, and FL-family methods reach 121~ms (the slight excess over the idealized 119~ms is due to optimizer, aggregation, and FedProx/SCAFFOLD per-parameter control terms). Standard SL is the sole outlier at 393~ms because node-side work is serialized across clients, 3.3$\times$ the parallel floor. Secure TL++ adds negligible critical-path arithmetic overhead ($\leq$0.20~ms, rendered at minimum height in the figure), but it incurs additional helper-side device time that runs in parallel and does not extend the critical path: 88.6~ms for secure cut~1 and 45.1~ms for secure cut~2, reflected by the diamond markers at $\approx$210~ms and $\approx$165~ms above those bars.

Taken together, Figures~\ref{fig:efficiency} and~\ref{fig:computation} give a more conservative deployment picture than activation-share estimates alone. TL++ base cut~1 is the strongest trusted configuration: 20.9~MB wire payload, 168~ms network time, 120~ms computation time, and near-centralized accuracy. Secure cut~1 is the most communication-efficient approximate-secure choice at 64.2~MB and 304~ms, with a 120~ms computation path and negligible ($\leq$0.03~ms) critical-path protocol overhead; it does incur 88.6~ms of additional helper device time that runs off the critical path. Exact secure cut~3 satisfies Proposition~\ref{prop:secure_equiv} but reaches 560~MB and 2{,}977~ms network time under per-batch synchronization---comparable to SCAFFOLD---because the large node-side state at this cut drives both the node-model sync and the helper relay. Practical exact-secure deployments therefore require synchronization amortization, sparse or quantized node updates, or parameter-efficient node-side training to bring cut~3's network cost into the cut-layer regime.

\section{Discussion}
\label{sec:discussion}

\subsection{Interpreting the Utility--Communication Trade-off}
\label{sec:discussion_results}

The experiments support a cut-dependent conclusion, not a claim that every TL++ configuration is communication-efficient. In trusted settings, base cut~1 is the strongest reported point: it has the best base accuracy and the smallest measured total payload; base cut~2 remains useful when a deeper split is desired. In secure settings, correctness comes first: cut~3 is the exact secure configuration for the evaluated lightweight additive-sharing protocol because the operations evaluated independently on shares above that cut are linear/affine and therefore compatible with sharewise additive decomposition. Its bottleneck is not the cut tensor but the large node-side state moved under per-batch synchronization.

A practical decision rule follows. Use base cut~1 or cut~2 when the orchestrator is trusted and bandwidth is the main constraint. Use secure cut~3 when the orchestrator and helper are individually semi-honest but non-colluding and exact additive-sharing correctness is required. If neither trust assumption is acceptable, TL++ should be combined with stronger MPC or differential-privacy mechanisms rather than presented as complete protection. The best cut should be re-evaluated for each architecture because activation size, trainable-parameter placement, and synchronization frequency jointly determine payload.

\subsection{Label Sharing and U-Shaped Split Variants}
\label{sec:label_u_shape}

The evaluated protocol follows the common server-side-loss design in SL~\cite{gupta_distributed_2018,vepakomma_split_2018}: nodes keep raw inputs local but provide labels for the selected virtual-batch samples to the orchestrator. This preserves the centralized-gradient interpretation, but it is a separate privacy assumption. In applications where labels are sensitive, secure TL++ should be read narrowly: it protects cut-layer activations and gradients under the semi-honest non-collusion assumption, but it does not hide labels, virtual-batch membership metadata, output predictions, or label-dependent loss gradients from every system component.

A U-shaped split could keep labels at nodes by returning server representations or logits for local loss computation, after which nodes send output-gradient signals. This aligns with label-sensitive SL variants~\cite{vepakomma_split_2018,sav_cure_2024,cheon_homomorphic_2017}, but adds a round trip, stricter sample-order bookkeeping, less orchestrator visibility into loss, and a harder choice over revealing or secret-sharing label-dependent output gradients. We therefore treat label sharing as an explicit deployment assumption and leave U-shaped TL++ for future work.

\subsection{Security Scope and Assumptions}
\label{sec:privacy_analysis}

Secure TL++ protects cut-layer activations and gradients from any single semi-honest server under the non-collusion model of Section~\ref{sec:pre}. Each activation is split by Equation~\eqref{eq:sharing}, so neither the orchestrator nor helper can reconstruct plaintext alone. The helper is the independent second-share holder, replacing single-server trust with non-collusion. The guarantee is limited: exact lightweight correctness requires linear/affine sharewise operations (cut~3 in CIFAR-10), output predictions are reconstructed for loss, and labels plus label-dependent gradients remain visible under the baseline protocol.

The optional Gaussian perturbations $\sigma_{\text{act}}$ and $\sigma_{\text{grad}}$ were not evaluated in this study. End-to-end differential-privacy claims would require clipping or another sensitivity bound, a privacy accountant, and accuracy measurements under the selected noise levels~\cite{abadi_deep_2016,mironov_renyi_2017}. Additive sharing also does not address colluding servers, malicious protocol deviations, compromised nodes, membership inference from final outputs~\cite{shokri_membership_2017}, reconstruction under side information~\cite{geiping_inverting_2020}, or timing and metadata side channels. Stronger deployments would need authenticated communication, auditing, malicious-secure MPC checks, and operational separation between orchestrator and helper.

\subsection{Scalability, Limitations, and Future Work}
\label{sec:limitations}

The main evaluation limitation is that accuracy measurements are combined with deterministic cost accounting rather than deployed packet traces. The model includes serialization/framing overhead and a link-latency/bandwidth scenario, but it does not report measured bytes, graphics processing unit (GPU) utilization, memory pressure, straggler behavior, or end-to-end multi-host throughput.

Two technical limitations are central. First, the exact lightweight secure path requires linear/affine sharewise operations; nonlinear server heads need polynomial approximations, garbled circuits, secure comparison, or hybrid MPC~\cite{yao_protocols_1982,mohassel_secureml_2017,wagh_securenn_2019,mohassel_aby3_2018,kumar_cryptflow_2020}. Second, secure cut~3 payload is dominated by node-side state under per-batch synchronization. Future systems work should test amortized synchronization, quantization, sparse or low-rank node updates, local-update schedules, larger node counts, heterogeneous client speeds, deadline-based virtual batches, leakage measurements, and privacy-noise sweeps with formal accounting.

Another direction is to relax the backward path. Forward-Forward (FF) propagation trains layers with local goodness objectives, including recent distributed FF variants~\cite{hinton_forward-forward_2022,aktemur_going_2024}, while Direct Feedback Alignment (DFA) uses direct random feedback from output errors~\cite{nokland_direct_2016}. These could reduce cut-layer gradient information returned to nodes and weaken sequential node--server dependencies, but they change the optimization objective and may affect TL's centralized-gradient equivalence. FF- or DFA-based TL++ variants should therefore be evaluated with the same accuracy, leakage, and communication accounting.

The workload scope is likewise limited. CIFAR-10 and PubMedQA give image and biomedical-language case studies, but not large-scale vision, speech, recommender, tabular medical, or long-context generative workloads. Future evaluations should include architectures with naturally linear heads, architectures requiring secure nonlinear computation, and tasks whose intermediate activations are tested against reconstruction attacks~\cite{he_model_2019,pasquini_unleashing_2021,zhu_deep_2019}.

\subsection{Deployment Implications}
\label{sec:deployment_implications}

TL++ is most attractive when communication is expensive relative to local computation and data heterogeneity makes local-update FL unreliable. In trusted infrastructure, base TL++ keeps raw data local, avoids full-model exchange, and achieves near-centralized accuracy in the reported experiments. In semi-honest multi-server deployments, secure TL++ prevents any single server from observing plaintext intermediate representations, but designers must either choose a cut whose independently evaluated sharewise operations are linear/affine or add stronger MPC components for nonlinear server operations. If clients are highly resource-constrained or networks are latency-dominated rather than bandwidth-limited, payload reduction alone may not improve training time; payload, latency, memory, and accuracy should be reported together.

\section{Conclusion}
\label{sec:conclusion}

This paper presented TL++, a traversal-learning framework combining virtual-batch optimization with optional additive-secret-sharing secure computation. Unlike conventional FL, TL++ exchanges cut-layer activations and gradients rather than full model updates, reducing communication while keeping raw data local. Secure TL++ adds a non-colluding helper and additive sharing so cut-layer activations and gradients are not exposed to any single server under the semi-honest non-collusion assumption.
Experiments on CIFAR-10 and BioGPT/PubMedQA show that TL++ can approach centralized accuracy while trading off optimization fidelity, communication cost, and privacy. Cost analysis shows that cut~1 and cut~2 achieve the lowest per-step payload in trusted and approximate-secure settings, with reductions up to 13.1$\times$ relative to full-model synchronization. Secure cut~3 satisfies exact additive sharing in the CNN implementation but has FL-level payload under per-batch synchronization, motivating amortization or parameter-efficient updates. TL++ therefore offers a middle ground between FL and SL for distributed AI when its trust, label-disclosure, and exact-share-evaluation assumptions hold. Future work will study runtime efficiency, stronger privacy, larger transformers, secure nonlinear layers, FF/DFA credit assignment, and personalized TL.

\section*{Data Availability}
This work uses two publicly available datasets:
\begin{itemize}
    \item \textbf{CIFAR-10}~\cite{krizhevsky_learning_2009}: VGG-style CNN experiments under extreme non-IID partitioning; \url{https://www.cs.toronto.edu/~kriz/cifar.html}.
    \item \textbf{PubMedQA}~\cite{jin_pubmedqa_2019}: BioGPT~\cite{luo_biogpt_2022} fine-tuning with LoRA~\cite{hu_lora_2021}; \url{https://pubmedqa.github.io}.
\end{itemize}

\section*{Code Availability}
TL++ source code is available under the MIT License:
\begin{itemize}
    \item \textbf{Main Framework}: VGG-style CNN on CIFAR-10; \url{https://github.com/neouly-inc/TLplus}.
    \item \textbf{Medical Case Study}: BioGPT~\cite{luo_biogpt_2022} LoRA fine-tuning on PubMedQA~\cite{jin_pubmedqa_2019}; \url{https://github.com/neouly-inc/TLplus-BioGPT-LoRa}.
\end{itemize}

\section*{Acknowledgments}
This work was partly supported by the Institute of Information \& Communications Technology Planning \& Evaluation (IITP) - ITRC (Information Technology Research Center) grant funded by the Korea government (MSIT) (IITP-2026-RS-2023-00259099) and by 2026 Hongik University Research Fund.

\bibliographystyle{ACM-Reference-Format}
\bibliography{references}

\end{document}